\documentclass{article}

\usepackage[preprint]{neurips_2019}

\usepackage{graphicx}
\usepackage[utf8]{inputenc} 
\usepackage[T1]{fontenc}    
\usepackage{hyperref}       
\usepackage{url}            
\usepackage{booktabs}       
\usepackage{amsfonts}       
\usepackage{nicefrac}       
\usepackage{microtype}      

\usepackage{ifpdf}
\usepackage{xcolor}
\usepackage{enumitem}
\usepackage[makeroom]{cancel}
\usepackage{algorithm,algorithmicx,algpseudocode}
\usepackage{calc}
\usepackage{framed}
\usepackage{algorithm,algorithmicx,algpseudocode}
\usepackage[T1]{fontenc}
\usepackage{mathtools}

\usepackage{amsmath,amssymb,amsfonts}        
\usepackage{amsthm}
\usepackage{pdflscape}
\usepackage{subcaption}

\algnewcommand{\MyState}[1]{\State
\parbox[t]{\dimexpr\linewidth-\ALG@thistlm}{\hangindent=0pt\strut\hangafter=1#1\strut}}
\makeatother
\algdef{SxN}{Myprocedure}{EndMyprocedure}[2]{\algorithmicprocedure\ \textproc{#1}\ifthenelse{\equal{#2}{}}{}{(#2)}}
\algdef{SxN}{MyFor}{MyEndFor}[1]{\algorithmicfor\ #1\ \algorithmicdo}

\newcommand{\comment}[1]{}

\newif\iffull
\fullfalse

\newcommand{\visiblelabel}[1]{}
\newcommand{\AlgorithmName}[1]{\label{alg:#1}\visiblelabel{#1}}

\newcommand{\ClaimName}[1]{\label{clm:#1}\visiblelabel{#1}}

\newcommand{\CorollaryName}[1]{\label{cor:#1}\visiblelabel{#1}}
\newcommand{\DefinitionName}[1]{\label{def:#1}\visiblelabel{#1}}
\newcommand{\EquationName}[1]{\label{eq:#1}\text{\visiblelabel{#1}}}

\newcommand{\FigureName}[1]{\label{fig:#1}\visiblelabel{#1}}
\newcommand{\LemmaName}[1]{\label{lem:#1}\visiblelabel{#1}}

\newcommand{\TheoremName}[1]{\label{thm:#1}\visiblelabel{#1}}

\newcommand{\SectionName}[1]{\label{sec:#1}\visiblelabel{#1}}
\newcommand{\SubsectionName}[1]{\label{sub:#1}\visiblelabel{#1}}

\newcommand{\Algorithm}[1]{Algorithm~\ref{alg:#1}}

\newcommand{\Claim}[1]{Claim~\ref{clm:#1}}

\newcommand{\Corollary}[1]{Corollary~\ref{cor:#1}}
\newcommand{\Definition}[1]{Definition~\ref{def:#1}}
\newcommand{\Equation}[1]{Eq.\:\eqref{eq:#1}}

\newcommand{\Figure}[1]{Figure~\ref{fig:#1}}
\newcommand{\Lemma}[1]{Lemma~\ref{lem:#1}}

\newcommand{\Section}[1]{Section~\ref{sec:#1}}
\newcommand{\Subsection}[1]{Subsection~\ref{sub:#1}}

\newcommand{\Theorem}[1]{Theorem~\ref{thm:#1}}

\newcommand{\SubsubsectionName}[1]{\label{subsub:#1}\visiblelabel{#1}}
\newcommand{\Subsubsection}[1]{Subsection~\ref{subsub:#1}}

\newtheorem{theorem}{Theorem}[section]
\newtheorem{remark}[theorem]{Remark}
\newtheorem{claim}[theorem]{Claim}
\newtheorem{corollary}{Corollary}[theorem]
\newtheorem{definition}[theorem]{Definition}
\newtheorem{lemma}[theorem]{Lemma}

\newenvironment{proofof}[1]{\begin{proof}[\rm\textbf{Proof} \,(of #1)]}{\end{proof}}

\newcommand{\bN}{\mathbb{N}}        

\newcommand{\bR}{\mathbb{R}}

\newcommand{\cC}{\mathcal{C}}

\newcommand{\cE}{\mathcal{E}}
\newcommand{\cF}{\mathcal{F}}

\newcommand{\cX}{\mathcal{X}}

\allowdisplaybreaks

\newcommand{\abs}[1]{\lvert #1 \rvert}
\newcommand{\Abs}[1]{\left\lvert #1 \right\rvert}

\newcommand{\setst}[2]{\left\{\; #1 \,:\, #2 \;\right\}}        
\newcommand{\set}[1]{\left \{ #1 \right \}}      

\newcommand{\prob}[1]{\operatorname{Pr}\left[\,#1\,\right]}               
\newcommand{\expect}[1]{\operatorname{E}\left[\,#1\,\right]}              
\newcommand{\expectg}[2]{\operatorname{E}\left[\,#1 \,\mid\, #2\,\right]} 

\newcommand{\transpose}{^{\textsf{T}}}

\newcommand{\norm}[1]{\left\lVert #1 \right\rVert}

\newcommand{\inner}[2]{\langle\: #1 ,\, #2 \:\rangle}

{\left(\begin{smallmatrix}}
{\end{smallmatrix}\right)}

\title{Simple and optimal high-probability bounds for \\ strongly-convex stochastic gradient descent}

\author{%
  Nicholas J. A.  Harvey \\
  Department of Computer Science\\
  University of British Columbia\\
  Vancouver, BC Canada \\
  \texttt{nickhar@cs.ubc.ca} \\
  \And
   Christopher Liaw \\
   Department of Computer Science \\
   University of British Columbia \\
   Vancouver, BC, Canada \\
   \texttt{cvliaw@cs.ubc.ca} \\
   \And
   Sikander Randhawa \\
   Department of Computer Science \\
   University of British Columbia \\
   Vancouver, BC, Canada \\
   \texttt{srand@cs.ubc.ca} \\
}

\begin{document}

\maketitle

\begin{abstract}
We consider stochastic gradient descent algorithms for minimizing a non-smooth, strongly-convex function.
Several forms of this algorithm, including suffix averaging, are known to achieve the optimal $O(1/T)$ convergence rate \emph{in expectation}.
We consider a simple, non-uniform averaging strategy of Lacoste-Julien et al. (2011) and prove that it achieves the optimal $O(1/T)$ convergence rate \emph{with high probability}. Our proof uses a recently developed generalization of Freedman's inequality.
Finally, we compare several of these algorithms experimentally and show that this non-uniform averaging strategy outperforms many standard techniques, and with smaller variance.
\end{abstract}

\section{Introduction}

Stochastic gradient descent (SGD) is perhaps the single most important algorithm for minimizing strongly convex loss functions. Its popularity is a combined consequence of the simplicity of its statement and its effectiveness in both theory and practice. Gradient descent is an iterative optimization procedure, where the current solution is updated by taking a step in the opposite direction of the current gradient. In the case of using SGD for Empirical Risk Minimization, the gradient of the loss function is often too expensive to compute. So instead, we select a data point uniformly at random and compute the gradient of the loss function using only this single data point. The resulting value is not necessarily a true gradient, but it is in expectation. When it is time to output a solution, the standard textbook choices are either to report the last iterate, or the average of iterates so far.

Surprisingly, there are situations where these textbook output strategies have provably sub-optimal performance, even though the algorithm is so popular. Here, by performance we are referring to the rate at which the loss of the output convergences to the true minimum value of the loss function. Consider the setting where the loss function is strongly convex, but not smooth (for example, the regularized SVM minimization problem). In the absence of smoothness, there is no guarantee that the value of the individual iterates of SGD improve over time. In fact, \citet{HLPR18} construct an example where the value of the iterates \emph{increases} over time.
Moreover, they show that the convergence rate of the individual iterates of even deterministic gradient descent is $\Omega(\log(T)/T),$ whereas the optimal rate is $\Theta(1/T)$ for a first-order algorithm. \citet{Rakhlin} show that returning the average of all of the iterates so far is also sub-optimal by a $\log(T)$ factor (this lower bound holds in expectation).

As a result, researchers have developed several algorithms which achieve the optimal $O(1/T)$ rate in expectation, some of which are simpler than others. 
Because the popularity of SGD is largely due to its simplicity, a straightforward variant of the algorithm attaining the optimal rate is significantly more desirable than some other, more complex procedure.
The non-uniform averaging strategy from \cite{NonUniformAvg} and the suffix-averaging strategy from \cite{Rakhlin} are likely the simplest and closest in resemblance to textbook statements of SGD. Each method runs standard SGD until output time. In \cite{NonUniformAvg}, a non-uniform average over all the iterates (with iterate $i$ given a weight proportional to $i$) is returned whereas in \cite{Rakhlin}, a uniform average over the last half of the iterates (referred to as the suffix-average) is returned. There still remains another important issue to resolve though, even provided the existence of simple algorithms which obtain the optimal expected $O(1/T)$ rate.

How many random trials are needed for suffix averaging or the non-uniform averaging strategies to actually achieve the $O(1/T)$ rate? (Recall that the error is random). Usual expositions of SGD provide bounds that hold in \emph{expectation}. This is a weak guarantee because it does not preclude the error of the algorithm from having large variance. Users of SGD want to be confident that the output of a single trial of the algorithm is extremely likely to provide the guaranteed convergence rate. In other words, they would prefer bounds that hold with high probability. Moreover, it is often impossible to run many trials of SGD and select the best one.
For example, considering the Empirical Risk Minimization setting, if we are dealing with many high dimensional data points, it can be prohibitively expensive to even evaluate the loss function.

It was shown recently in \cite{HLPR18}, that suffix-averaging obtains a convergence rate of $O(\log(1/\delta)/T)$ with probability at least $1-\delta.$  However, implementing suffix-averaging when the time horizon is not known ahead of time (for example, stopping SGD when the norm of the gradient is small) requires a modicum of care. Non uniform averaging could be an equally attractive alternative if its convergence rate held with high probability.

\paragraph{Main theoretical results. }We show that running standard SGD and returning the very simple non-uniform average of the iterates from \cite{NonUniformAvg} has error at most $O(\log(1/\delta)/T)$ with probability $1-\delta$. The analysis is simple and exposes a martingale which satisfies a special recursive property which was also observed in \cite{HLPR18}. It is intriguing that this recursive property arises in multiple settings when analyzing SGD. Moreover, we show a matching lower bound of $\Omega(\log(1/\delta)/T)$ with probability at least $\delta.$ The analysis uses the simple univariate function $\frac{1}{2}x^2$.
Thus, we have a tight high probability analysis of a very simple output strategy for SGD which attains the optimal rate. 

\paragraph{Experimental results. }In addition, we run detailed experiments for various return schemes of SGD for SVMs on synthetic data and real-world datasets.
Our experimental results strongly suggest that 
the suffix averaging and non-uniform averaging schemes 
should be preferred over the final iterate and uniform averaging schemes.

\subsection{Related Work}
\SubsectionName{RelatedWork}
There are a number of other algorithms which obtain the optimal $O(1/T)$ convergence rate amongst first-order methods for minimizing a non-smooth, strongly-convex function.
\citet{HK14} proved that a variant of SGD, called Epoch-GD obtains the optimal rate.
Here, they partition the total time $T$ into exponentially growing epochs.
Within each epoch, they run standard SGD (with an appropriate step size) and after the end of each epoch, they reset the current point to the average of the iterates in the previous epoch.

Later, \citet{Rakhlin} and \citet{NonUniformAvg} independently discovered simpler algorithms that also achieve the optimal $O(1/T)$ rate.
In fact, both algorithms run standard SGD with the standard step size proportional to $1/t$; the only difference between the two algorithms is the return value of the algorithm.
In \citet{Rakhlin}, they show that suffix averaging, where one returns the last $\alpha$ fraction of the iterates (for some constant $\alpha > 0$), achieves the optimal rate.
On the other hand, \citet{NonUniformAvg} prove that a certain non-uniform average (see \Algorithm{StrongLipschitzPGD}) of the iterates also achieves the optimal rate.
One advantage of non-uniform averaging is that the iterates can be easily computed on the fly.

Recently, \cite{NesterovShikhman} have devised a modification of gradient descent for which the error of the \emph{last iterate} converges at the optimal rate.
Even more recently, \citet{jain2019making} showed that even for unmodified gradient descent the last iterate can be made to achieve the optimal rate, 
if the time horizon is known beforehand, and if the step-size is chosen carefully using the time horizon.
Interestingly, they also show that knowing (or having a bound on) the time horizon is necessary for all the individual iterates to achieve the optimal rate.

\paragraph{High-probability upper bounds.}
All the results stated above hold only in expectation and do not rule out the possibility that the return value has high variance.
Moreover, it can be expensive to compute the objective value of a point.
Hence, it is desirable to have a high-probability upper bound on the return value.

To assist in this task, \citet{HLPR18} recently developed a generalization of Freedman's Inequality.
Using this, they show that if one runs SGD with the standard $1/t$ step sizes, then the last iterate and the suffix average schemes achieve error $O(\log(T)/T)$, and $O(1/T)$, respectively, with high probability.
Using similar methods, \citet{jain2019making} prove that the last iterate of SGD with carefully chosen step sizes achieves an error of $O(1/T)$.
(As mentioned above, this requires advance knowledge of $T$.)
The uniform average was earlier shown by \cite{KT08} to achieve error $O(\log(T)/T)$ with high probability.
Here, we will also employ the generalized Freedman's Inequality to prove a tight high-probability upper bound on the non-uniform averaging scheme.
\section{Preliminaries}
Let $\cX$ be a closed, convex subset of $\bR^n$, $f \colon \cX \to \bR$ be a convex function, and $\partial f(x)$ be the subdifferential of $f$ at $x$.
Our goal is to solve the convex program $\min_{x \in \cX} f(x)$.
We assume that $f$ may not be explicitly represented.
Instead, the algorithm is allowed to query $f$ via a stochastic gradient oracle, i.e.,~if the oracle is queried at $x$ then it returns $\hat{g} = g - \hat{z}$ where $g \in \partial f(x)$ and $\expect{\hat{z}} = 0$ conditioned on all past calls to the oracle.
Furthermore, we assume that $f$ is $L$-Lipschitz, i.e.~$\norm{g} \leq L$ for all $x \in \cX$ and $g \in \partial f(x)$ and that $f$ is $\mu$-strongly convex, i.e.
\begin{equation}
\EquationName{strongly_convex_def}
f(y) ~\geq~ f(x) + \inner{g}{y-x} + \frac{\mu}{2}\norm{y - x}^2 \quad \forall y, x \in \cX, g \in \partial f(x).
\end{equation}
Throughout this paper, $\norm{\cdot}$ denotes the \emph{Euclidean} norm in $\bR^n$, $\Pi_{\cX}$ denotes the projection operator onto $\cX$ and $[T]$ denotes the set $\set{1, \ldots, T}$.
For the sake of simplicity, we assume that $\norm{\hat{z}} \leq 1$ a.s.

In this paper, we analyze SGD with the averaging scheme proposed by \citet{NonUniformAvg}.
The algorithm is given in \Algorithm{StrongLipschitzPGD}.

\begin{algorithm}
\caption{Stochastic, projected gradient descent for minimizing a $\mu$-strongly convex, $L$-Lipschitz
function with an unknown time horizon.}
\AlgorithmName{StrongLipschitzPGD}
\begin{algorithmic}[1]
\Myprocedure{ProjectedGradientDescent}{$\cX \subseteq \bR^n$,\, $x_1 \in \cX$}
\MyFor{$t \leftarrow 1,\ldots,T$}
    \State Let $\eta_t = \frac{2}{\mu (t+1)}$
    \State $y_{t+1} \leftarrow x_t - \eta_t \hat{g}_t$,
    where $\expectg{\hat{g}_t}{\hat{g}_1, \ldots, \hat{g}_{t-1}} \in \partial f(x_t)$
    \State $x_{t+1} \leftarrow \Pi_\cX(y_{t+1})$
\MyEndFor
\State \textbf{return} $\sum_{t=1}^T \frac{t}{T(T+1)/2} x_t $
\EndMyprocedure
\end{algorithmic}
\end{algorithm}

Finally, we will use $\cF_t$ to denote the $\sigma$-field generated by
the random vectors $\hat{g}_1, \ldots, \hat{g}_t$.

\begin{remark}
    As noted in \cite{NonUniformAvg}, the return value of \Algorithm{StrongLipschitzPGD} can be computed in an online manner.
    Indeed, we can set $z_1 = x_1$, and we can set $z_t = \rho_t x_t + (1-\rho_t)z_{t-1}$  for $t \geq 2$, where $\rho_t = \frac{2}{t+1}$.
    It is a straightforward calculation to check that $z_T = \sum_{t=1}^T \frac{t}{T(T+1)/2} x_t$.
\end{remark}

\subsection{Probability tools}
Our main probabilistic tool is an extension of Freedman's Inequality \citep{F75} developed recently by \citet{HLPR18}.
Roughly speaking, Freedman's Inequality asserts that a martingale is bounded by the square root of its total conditional variance (TCV).
As we shall see in the sequel, the martingales that arise from analyzing SGD exhibit a ``chicken-and-egg'' phenomenon wherein the TCV of the martingale is bounded by (a linear transformation of) the martingale itself.
Here, we state a specialized form of the Generalized Freedman's Inequality which is a simple corollary from the statement given in \cite{HLPR18}.
\newcommand{\fantwo}{
Let $\{ d_t, \cF_t \}_{t=1}^T$ be a martingale difference sequence.
Suppose that, for $t \in [T]$, $v_{t-1}$ are non-negative  $\cF_{t-1}$-measurable random variables satisfying $\expectg{\exp(\lambda d_t)}{\cF_{t-1}} \leq \exp\left(\frac{\lambda^2}{2} v_{t-1}\right)$ for all $\lambda > 0$.
Let $S_T = \sum_{t=1}^T d_t$ and $V_T = \sum_{t=1}^T v_{t-1}$.
Suppose there exists $\alpha_1, \ldots, \alpha_T, \beta \in \bR_{\geq 0}$ such that $V_T \leq \sum_{t=1}^T \alpha_t d_t + \beta$.
Let $\alpha \geq \max_{t \in [T]} \alpha_t$.
Then
\[
\prob{S_T \geq x} ~\leq~ \exp\left( -\frac{x^2}{4\alpha\cdot x + 8\beta } \right).
\]
}

\begin{theorem}[Generalized Freedman, {\cite[Theorem~3.3]{HLPR18}}]
\TheoremName{FanV2}
\fantwo
\end{theorem}
\section{Main results}

Our main result is a high-probability upper bound on the final iterate of \Algorithm{StrongLipschitzPGD}. The proof is given in \Section{ub}.

\begin{theorem}
\TheoremName{StochasticStrongPGD}
Let $\cX \subseteq \bR^n$ be a convex set.
Suppose that $f : \cX \rightarrow \bR$ is $\mu$-strongly convex
(with respect to $\norm{\cdot}_2$) and $L$-Lipschitz.
Assume that:
\begin{enumerate}[label=(\emph{\alph*})]
\item $g_t \in \partial f(x_t)$ for all $t$ (with probability $1$).
\item $\norm{\hat{z}_t} \leq 1$ (with probability $1$).
\end{enumerate}
Set $\eta_t = \frac{2}{\mu (t+1)}$ and $\gamma_t = \frac{t}{T(T+1)/2}$.
Then, for any $\delta \in (0,1)$, with probability at least $1-\delta$,
$$
f\left( \sum_{t=1}^T \gamma_t x_t \right) - f(x^*)
    ~\leq~  O  \left (\frac{L \cdot \log (1/\delta) + L^2}{ \mu}  \cdot \frac{1}{T} \right ). 
$$
\end{theorem}

\begin{remark}
It is possible to strengthen the statement of \Theorem{StochasticStrongPGD} by replacing assumption (a) with the weaker assumption that $\norm{\hat{z}_t}$ is subgaussian conditioned on $\cF_{t-1}$ (for example, $\hat{z}_t \sim N(0, \frac{1}{n} I_n)$). A more detailed discussion can be found in the supplementary material.
\end{remark}

We also show that the bound in \Theorem{StochasticStrongPGD} is tight up to constant factors.
The proof is in \Section{LB}.

\begin{claim}
\ClaimName{LB}
Suppose $ \sqrt{6} \leq  \frac{\sqrt{2\log(1/\delta)}}{3}  \leq \sqrt{T}/4$.
There exists a sub-gradient oracle such that running \Algorithm{StrongLipschitzPGD} on the function $f(x) = \frac{x^2}{2}$ with step sizes $\eta_t = \frac{1}{t+1}$ satisfies the following. With probability at least $\delta$
\[  f\left(   \sum_{t=1}^T \gamma_t x_t \right) - f(x^*) \geq \frac{ \log(1/\delta)}{9 \cdot T},  \]
where $\gamma_t = \frac{t}{T(T+1)/2}.$
\end{claim}

\section{Proof of high probability upper bound}
\SectionName{ub}

The proof follows that of \cite{NonUniformAvg} but we must be careful with the noise terms as our goal is obtain a high probability bound.
We will need one technical lemma whose proof we relegate to the next subsection.
\begin{lemma}
\LemmaName{StronglyConvexSGDBound1}
Let $Z_T = \sum_{t=1}^T t \cdot \inner{\hat{z}_t}{x_t - x^*}.$ Then for any $\delta \in (0,1)$, $Z_T \leq O \left ( \frac{L}{\mu} \cdot T \log(1/\delta) \right )$, with probability at least $1-\delta$.
\end{lemma}
\begin{proofof}{\Theorem{StochasticStrongPGD}}
Define $\hat{z}_t = g_t - \hat{g}_t$.
Since $f$ is $\mu$-strongly convex, we have
\begin{align*}
    f(x_t) - f(x^*)
        &~\leq~ \inner{ g_t }{ x_t - x^* }
            - \frac{\mu}{2} \norm{x_t - x^*}_2^2 \\
        &~=~ \inner{ \hat{g}_t }{ x_t - x^* }
            - \frac{\mu}{2} \norm{x_t - x^*}_2^2
            + \inner{ \hat{z}_t }{ x_t - x^* }.
\end{align*}
The first two terms can be bounded as follows.
\begin{align*}
        & \inner{ \hat{g}_t }{ x_t - x^* }
        - \frac{\mu}{2} \norm{x_t - x^*}_2^2 \\
    &~=~ \frac{1}{\eta_t} \inner{ x_t - y_{t+1} }{ x_t - x^* }
        - \frac{\mu}{2} \norm{x_t - x^*}_2^2
        \qquad\text{(by the gradient step)}\\
    &~=~ \frac{1}{2 \eta_t} \Big(
            \norm{x_t - y_{t+1}}_2^2
          + \norm{x_t     - x^*}_2^2
          - \norm{y_{t+1} - x^*}_2^2
        \Big)
        - \frac{\mu}{2} \norm{x_t - x^*}_2^2 \\
    &~\leq~ \frac{1}{2 \eta_t} \Big(
            \norm{x_t - y_{t+1}}_2^2
          + \norm{x_t     - x^*}_2^2
          - \norm{x_{t+1} - x^*}_2^2
        \Big)
        - \frac{\mu}{2} \norm{x_t - x^*}_2^2.
\end{align*}
The last line uses a property of Euclidean projections:
since $x_{t+1}$ is the projected point $\Pi_\cX(y_{t+1})$ and $x^* \in \cX$,
we have $\norm{x_{t+1} - x^*}_2^2 \leq \norm{y_{t+1} - x^*}_2^2$.

It is convenient to scale by $t$ in order to later obtain a telescoping sum.
Using the definition of the gradient step, i.e.~$x_t - y_{t+1} = \eta_t \hat{g}_t$, we have
\begin{align*}
& t \cdot \Big(f(x_t) - f(x^*)
 - \inner{ \hat{z}_t }{ x_t - x^* } \Big) \\
    &~\leq~ \frac{t \norm{\eta_t \hat{g}_t}_2^2}{2 \eta_t}
          + t \Big(\frac{1}{2 \eta_t} - \frac{\mu}{2}\Big) \norm{x_t-x^*}_2^2
          - \frac{t}{2 \eta_t} \norm{x_{t+1} - x^*}_2^2
          \\
    &~=~ \frac{t \norm{\hat{g}_t}_2^2}{\mu (t+1)}
          + \Big(\frac{\mu t(t+1)}{4} - \frac{2 \mu t}{4}\Big) \norm{x_t-x^*}_2^2
          - \frac{t (t+1) \mu}{4} \norm{x_{t+1} - x^*}_2^2
         \\
    &~\leq~ \frac{(L+1)^2}{\mu}
          + \frac{\mu}{4} \cdot\Big( t(t-1) \norm{x_t-x^*}_2^2
          - t (t+1) \norm{x_{t+1} - x^*}_2^2 \Big).
\end{align*}
Now, summing over $t$, the right-hand side telescopes and we obtain
\begin{align*}
\sum_{t=1}^T t \cdot \big(f(x_t) - f(x^*)\big)
    ~\leq~ \sum_{t=1}^T t \cdot \inner{ \hat{z}_t }{ x_t - x^* } 
        + \frac{T \cdot (L+1)^2}{\mu}
\end{align*}
Dividing by $T(T+1)/2$ and applying Jensen's inequality, we obtain
\begin{align*}
f\Big( \sum_{t=1}^T \gamma_t x_t \Big) - f(x^*)
    &~\leq~ \sum_{t=1}^T \gamma_t \cdot \big(f(x_t) - f(x^*)\big) \\
    &~\leq~ \frac{2}{T(T+1)} \underbrace{\sum_{t=1}^T t \cdot \inner{ \hat{z}_t }{ x_t - x^* }}_{=:\, Z_T}
        + \frac{2 \cdot (L+1)^2}{\mu(T+1)}.
\end{align*}
Finally, we can use \Lemma{StronglyConvexSGDBound1} to obtain a high probability bound on $Z_T$, completing the proof of the theorem.
\end{proofof}

\subsection{Bounding $Z_T$}
\SubsectionName{BoundingZ}
Observe that $Z_T$ is a sum of a martingale difference sequence.
Define $d_t = t \cdot \inner{\hat{z}_t}{x_t - x^*}$, $v_{t-1} := t^2 \norm{x_t - x^*}$, and $V_T = \sum_{t=1}^T v_{t-1}$.
Note that $v_{t-1}$ is $\cF_{t-1}$-measurable. The next claim shows that $v_{t-1}$ and $d_t$ satisfy the assumptions of Generalized Freedman's inequality (\Theorem{FanV2}).
\begin{claim}
\ClaimName{BIsTCV}
For all $t$ and $\lambda > 0$, we have
$\expectg{\exp \left ( \lambda d_t \right )}{\cF_{t-1}} \leq \exp \left ( \frac{\lambda^2}{2} v_{t-1}  \right )$.
\end{claim}
\begin{proof}
    First, we can apply Cauchy-Schwarz to get that $\abs{t \inner{\hat{z}_t}{x_t - x^*}} \leq t \cdot\norm{\hat{z}_t} \cdot \norm{x_t - x^*} \leq t\cdot\norm{x_t - x^*}$ because $\norm{\hat{z}_t} \leq 1$ a.s.
    Next, applying Hoeffding's Lemma (\cite[Lemma~2.6]{massart2007concentration}),
    we have $\expectg{\exp \left( \lambda t \inner{\hat{z}_t}{x_t - x^*} \right )}{\cF_{t-1}} \leq \exp \left ( \frac{\lambda^2}{2} t^2 \norm{x_t - x^*}^2  \right )$.
\end{proof}

To bound $Z_T$, we will show that we can bound its TCV by a linear combination of the increments.
This will allow us to use the Generalized Freedman Inequality (\Theorem{FanV2}).
\begin{lemma}
\LemmaName{BoundingB}
There exists non-negative constants $\alpha_1, \ldots, \alpha_T$ such that $\max_{i \in [T]} \set{\alpha_i} = O \left (\frac{T}{\mu} \right)$ and $\beta = O \left (  \frac{L^2}{\mu^2} T^2 \right )$ such that
$V_T \leq \sum_{t=1}^T \alpha_t d_t + \beta$.
\end{lemma}
\begin{proofof}{\Lemma{StronglyConvexSGDBound1}}
By \Claim{BIsTCV}, we have $\expectg{\exp(\lambda d_t)}{\cF_{t-1}} \leq \exp\left( \frac{\lambda^2}{2} v_{t-1} \right)$ for all $\lambda > 0$.
By \Lemma{BoundingB}, we have $V_T \leq \sum_{t=1}^T \alpha_t d_t + \beta$.
Plugging $\alpha = O \left (\frac{T}{\mu} \right)$, $\beta = O \left (\frac{L^2}{\mu^2} T^2 \right )$, and $x = O \left  ( \frac{L}{\mu} \cdot T \log(1/\delta) \right )$ into \Theorem{FanV2} proves the lemma.
\end{proofof}

It remains to prove \Lemma{BoundingB}.
To do so, we will need the following two lemmata, which are adapted from \cite{Rakhlin} to use the step sizes $\eta_t = \frac{2}{\mu (t+1)}$. For completeness, we provide a proof in the supplementary material.
\begin{lemma}[{\protect\cite[Lemma 5]{Rakhlin}}]
\LemmaName{RakhlinDiameterBound}
With probability $1$, and for all $t$, $\norm{x_t-x^*} ~\leq~ \frac{2 L}{\mu}.$
\end{lemma}

\begin{lemma}[{\protect\cite[Lemma 6]{Rakhlin}}]
\LemmaName{RakhlinWBound}For all $t \geq 3$, there exists non-negative numbers $a_1(t), \ldots , a_t(t)$ with $a_i(t) = \Theta(i^3/t^4)$ and $b_1(t), \ldots, b_t(t)$ with $b_i(t) = \Theta(i^2/t^4)$, such that with probability $1$
$$
\norm{x_{t+1}-x^*}^2 ~\leq~ \frac{4}{\mu}
    \sum_{i=3}^t a_i(t) \inner{ \hat{z}_i }{ x_i - x^* }
    + \frac{4}{\mu^2} \sum_{i=3}^t b_i(t) \norm{\hat{g}_t}^2.
$$
\end{lemma}
\begin{remark}
\Lemma{RakhlinDiameterBound} and \Lemma{RakhlinWBound} are true regardless of the assumption we place on $\hat{z}_t$. 
\end{remark}

\begin{proofof}{\Lemma{BoundingB}}
Recall $\norm{\hat{g}_i} \leq L +1$ because $f$ is $L$-Lipschitz and $\norm{\hat{z}_i}\leq 1$ almost surely. By \Lemma{RakhlinDiameterBound} and \Lemma{RakhlinWBound}, we have: 
\begin{align*}
    &V_T \\
        &~=~ \sum_{t=1}^T  t^2 \cdot  \norm{x_t - x^*}^2 \\
        &~\leq~   \frac{56L^2}{\mu^2} + \sum_{t=4}^T t^2  \left (   \frac{4}{\mu} \sum_{i=3}^{t-1} a_i(t-1)\inner{\hat{z}_i}{x_i - x^*} + \frac{4}{\mu^2} \sum_{i=3}^{t-1} b_i(t-1) \norm{\hat{g}_i}^2 \right ) \\
        &~\leq~   \frac{56L^2}{\mu^2} + \sum_{t=4}^T t^2  \left (   \frac{4}{\mu} \sum_{i=3}^{t-1} a_i(t-1)\inner{\hat{z}_i}{x_i - x^*} + \frac{4(L+1)^2}{\mu^2} \sum_{i=3}^{t-1} b_i(t-1)  \right ) \\
        &~=~ \frac{4}{\mu} \sum_{t=4}^T t^2 \left ( \sum_{i=3}^{t-1} a_i(t-1) \inner{\hat{z}_i}{x_i - x^*} \right ) + \frac{4(L+1)^2}{\mu^2} \sum_{t=4}^T t^2 \left (\sum_{i=3}^{t-1}b_i(t-1) \right ) + \frac{56L^2}{\mu^2} \\
        &~=~ \sum_{i=3}^{T-1}  \underbrace{ \frac{4}{\mu} \left ( \sum_{t=i+1}^T t^2 \cdot \frac{a_i(t-1)}{i}\right )}_{\coloneqq \alpha_i} \cdot i  \inner{\hat{z}_i}{x_i - x^*} +  \underbrace{\frac{4(L+1)^2}{\mu^2} \sum_{t=4}^T t^2 \left (\sum_{i=3}^{t-1}b_i(t-1) \right ) + \frac{56L^2}{\mu^2}}_{\coloneqq \beta}\\
\end{align*}
Define $\alpha_1, \alpha_2, \alpha_T = 0.$ We have  $V_T \leq \sum_{i=1}^T \alpha_i \cdot i \cdot \inner{\hat{z}_i}{x_i -x^*} + \beta.$ It remains to show $\max \set{\alpha_i} = O \left ( \frac{T}{\mu}\right )$ and $\beta = O \left ( \frac{L^2}{\mu^2} T^2 \right ).$ To bound $\max \set{\alpha_i}$, observe that for $i \in \set{3, \ldots , T-1},$ 
\[\sum_{t=i+1}^T t^2 \cdot \frac{a_i(t-1)}{i} ~=~ \sum_{t=i+1}^T t^2 O \left (\frac{i^2}{t^4}\right ) ~=~ \sum_{t=i+1}^T t^2 O \left (\frac{1}{t^2}\right ) ~=~ O \left ( {T -i} \right ).  \]To bound $\beta,$ observe
\[ \sum_{t=4}^T t^2  \left ( \sum_{i=3}^{t-1}b_i(t-1)\right )~=~ \sum_{t=4}^T t^2  \sum_{i=3}^{t-1} O \left ( \frac{i^2}{t^4}  \right ) ~=~   \sum_{t=4}^T t^2  \sum_{i=3}^{t-1} O \left ( \frac{1}{t^2}\right ) ~=~ \sum_{t=4}^T O(t) ~=~ O(T^2).   \]
\end{proofof}

\section{Description of high probability lower bound}
\SectionName{LB}

\noindent\textbf{Setup of the lower bound. } Consider the one dimensional, $1$-strongly convex function $f(x) = \frac{1}{2}x^2$ with feasible region $\cX = [-6,6]$. Suppose, that at any point $x_t$, the gradient oracle returns a value of the form $x_t - \hat{z}_t,$ where $\expect{\hat{z}_t} = 0.$ Clearly, this is a valid subgradient oracle. Suppose we run \Algorithm{StrongLipschitzPGD} with a \emph{slightly modified step size of} $\eta_t = \frac{1}{t+1}$ starting from initial point $x_1 = 0$. 

\begin{remark}
Note that we are using a step size of $\frac{1}{t+1}$ instead of the step size $\frac{2}{t+1}$ used in the statement of \Algorithm{StrongLipschitzPGD}. It is possible to modify the analysis to use the step size as stated in \Algorithm{StrongLipschitzPGD}, however the analysis is much cleaner using $\frac{1}{t+1}$ and still captures the main ideas.
\end{remark}

\begin{claim}
\ClaimName{lowerBoundIterates}
Suppose $x_1 = 0$ and assume $\Abs{\hat{z}_t} \leq 6$. Then, $x_t = \frac{1}{t}\sum_{i=1}^{t-1} \hat{z}_i$ for all $2 \leq t \leq T$. 
\end{claim}

\noindent\textbf{Definition of gradient oracle.} Let $\hat{z}_t = 0$ if $t \leq \frac{T}{2}$ or $T > \frac{3T}{4}$ and otherwise for $ T/2 +1  \leq t \leq \frac{3T}{4}$, define $\hat{z}_t = \frac{T+1}{T-t} X_t$ where $X_t$ is uniform in $\set{+1,-1}.$ Note that this gradient oracle satisfies the conditions of \Claim{lowerBoundIterates}. That is $\abs{\hat{z}_t} \leq 6$ for all $t,$ as long as $T \geq 2.$

By definition of $\hat{z}_t$ and \Claim{lowerBoundIterates}, one can check that $\sum_{i=1}^T \gamma_t x_t$ is an average of Bernoulli random variables. Applying a reverse Chernoff bound from \cite{RevChernoff} completes the proof.
The complete details can be found in the supplementary materials.
\section{Experimental results}
\SectionName{experimental_results}
The four return strategies discussed in this paper have fairly similar theoretical guarantees. The aim of this section is to compare the strategies on real data sets, focusing on two aspects of their performance: the expectation and the concentration of the objective value. The results are shown in Figure~\ref{fig:iterates_plot_cina0_protein}. Additional experimental results can be found in the supplementary material (\Section{additional_experiments}).

The results of the experiments reveal a clear message. The final iterate and the uniform average return strategies perform noticeably worse than the suffix average and non-uniform average, both in terms of expectation and concentration. This is consistent with the fact that their theoretical guarantees are also worse. The performance of the suffix average and the non-uniform average are nearly indistinguishable, with the suffix average having a slight advantage in expectation.

\paragraph{Methodology. } We consider the regularized SVM optimization problem
\[  f(w) ~\coloneqq~ \frac{\lambda}{2} \norm{w}^2 +  \sum_{i=1}^m \max \set{ 0, 1- y_i w\transpose x_i }, \]where $m$ is the number of data points and we use $n$ to denote the dimension of each data point. We run SGD with step size $\eta_t = \frac{2}{t+1}$ and with regularization parameter $\lambda = 1/m$. This particular step size is required for \Theorem{StochasticStrongPGD}, and the analyses for the other averaging schemes can also accommodate this choice of step size. Furthermore, we found that the relative performance of the different averaging schemes is not particularly sensitive to the choice in step size. We plot the value of $f$ for each return strategy every $m$ iterations (which we refer to as an `effective pass'). Since the output of SGD is random, there is a distribution over the outputs which we would like to capture. We run 1000 trials of SGD. The colored curves are exactly these 1000 trials, which are plotted with low opacity. At any point in time, the darkness of the plot at a specific objective value indicates the number of trials that achieved that value at that time. The dotted dark lines represent the average amongst the trials. 

Figure~\ref{fig:iterates_plot_cina0_protein} suggests that practitioners should consider using the suffix average or non-uniform average in lieu of the final iterate or uniform average. It is possible to implement suffix averaging and non-uniform averaging with minimal effort, and the performance boost is significant. Implementing non-uniform averaging (even when the time horizon is not fixed ahead of time) only requires a single additional line of code. 

\paragraph{Data sets.} We performed our experiments on a set of freely available binary classification data sets. The experiments from this section use the \emph{cina0} ($n = 16033$ and $d = 132$) and the \emph{protein} ($m= 145751$ and $n = 74$) data sets. We ran the same experiments on the \emph{rcv1} ($m=20242$ and $n = 47236$), \emph{covtype} ($m = 581,012$ and $n=54$) and \emph{quantum} ($m=50000$ and $n=78)$ data sets. The results for these data sets can be found in \Section{additional_experiments}. Sparse features were scaled to $[0,1]$ whereas dense features were scaled to have zero mean and unit variance. Data sets \emph{quantum} and \emph{protein} can be found at the \href{http://osmot.cs.cornell.edu/kddcup/datasets.html}{\color{blue}{KDD cup 2004 website}}, \emph{cina0} can be found at the \href{http://www.causality.inf.ethz.ch/home.php}{\color{blue}{Causality Workbench website}} and \emph{covtype} and \emph{rcv1} can be found at the \href{https://www.csie.ntu.edu.tw/~cjlin/libsvmtools/datasets/}{\color{blue}{LIBSVM website}}.

\clearpage 
\begin{landscape}
\begin{figure}[h]
    \centering
    \begin{subfigure}[b]{\paperwidth}
        \centering
        \includegraphics[width=\paperwidth]{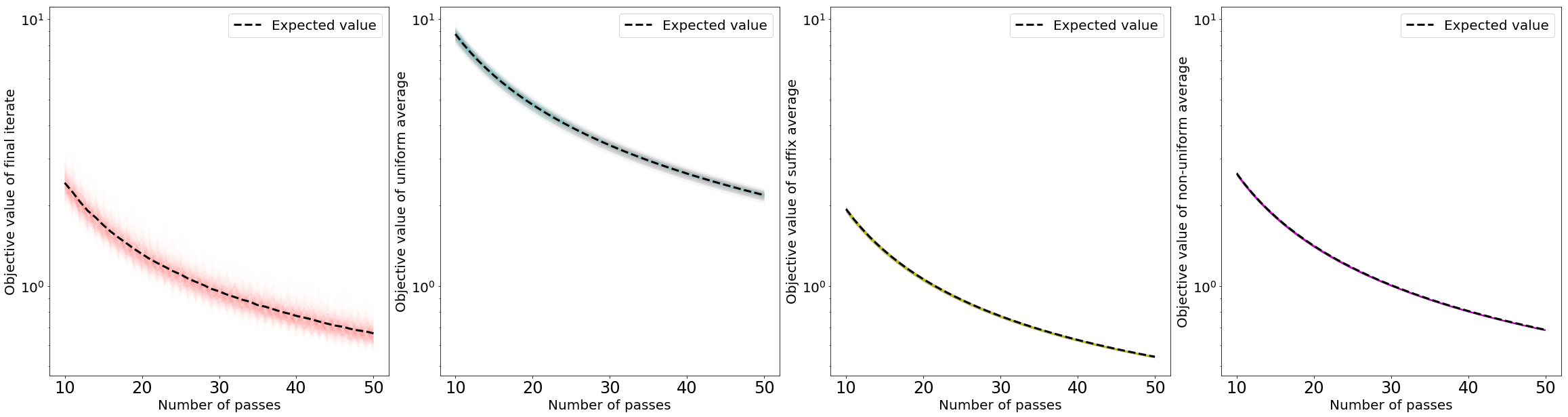}
    \caption{cina0}
    \end{subfigure} \\
    \begin{subfigure}[b]{\paperwidth}
        \centering
        \includegraphics[width=\paperwidth]{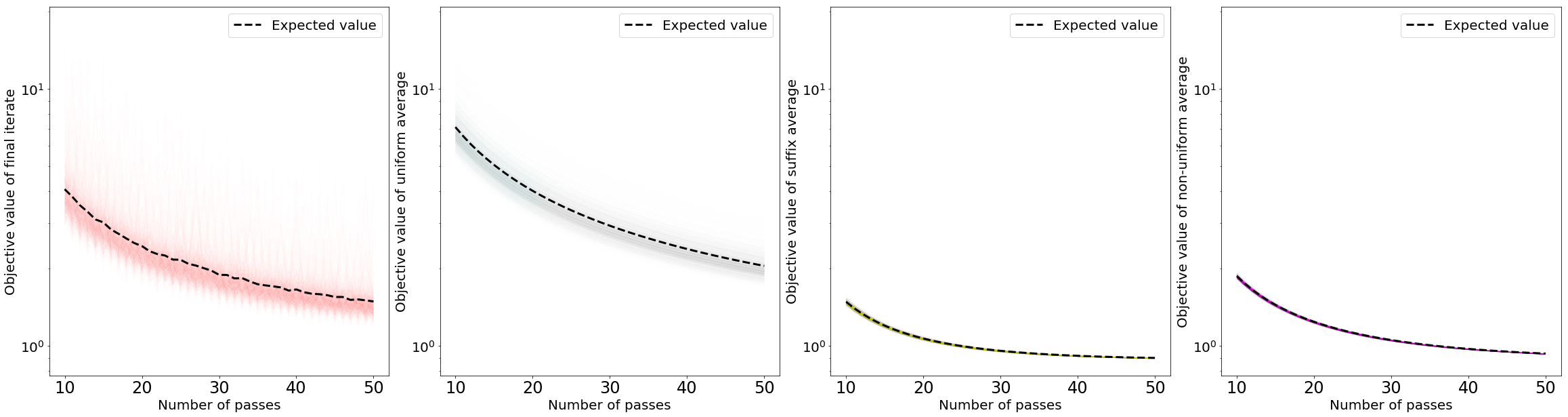}
    \caption{protein}
    \end{subfigure}
    \caption{Number of effective passes vs. objective value. The first row plots the results using the \emph{cina0} data set, whereas the second plots results using the \emph{protein} data set. From left to right, we plot the objective value over time of the final iterate, uniform average, suffix average and non-uniform average for 1000 trials of SGD.}
    \label{fig:iterates_plot_cina0_protein}
\end{figure}
\end{landscape}

\clearpage
\bibliographystyle{plainnat}
\bibliography{GD}

\clearpage

\appendix
\section{Proof of \Lemma{RakhlinDiameterBound} and \Lemma{RakhlinWBound}}
Both of the proofs in this section are slight modifications of the proofs found in \cite{Rakhlin}.
\begin{proofof}{\Lemma{RakhlinDiameterBound}}
 Due to strong convexity and the fact that $f(x_t) - f(x^*) \geq 0,$ we have 
\[  L \norm{x_t-x^*} \geq \norm{g_t}\norm{x_t - x^*} \geq \inner{g_t}{x_t - x^*} \geq \frac{\mu}{2}\norm{x_t - x^*}^2, \]where we used $L$-Lipschitzness of $f$ to bound $\norm{g_t}$ by $L.$

\end{proofof}

\begin{proofof}{\Lemma{RakhlinWBound}}
The definition of strong convexity yields
\[  \inner{g_t}{x_t - x^*} \geq f(x_t) - f(x^*) + \frac{\mu}{2}\norm{x_t - x^*}^2. \]Strong convexity and the fact that $0 \in \partial f(x^*)$ implies
\[  f(x_t) - f(x^*) \geq \frac{\mu}{2}\norm{x_t - x^*}^2. \]Next, recall that for any $x \in \cX,$ and for any $z$, we have $\norm{\Pi_{\cX} (z) - x} \leq \norm{z - x}.$ Lastly, recall $\eta_t = \frac{2}{\mu (t+1)}.$ Using these, we have
\begin{align}
    \norm{x_{t+1}- x^*}^2 
        &~=~ \norm{\Pi_{\cX} (  x_t - \eta_t \hat{g}_t ) - x^*}^2 \nonumber\\
        &~\leq~ \norm{x_t - \eta_t \hat{g}_t - x^*}^2\nonumber \\
        &~=~ \norm{x_t - x^*}^2 -2\eta_t \inner{\hat{g}_t}{x_t - x^*} + \eta_t^2 \norm{\hat{g}_t}^2\nonumber \\
        &~=~ \norm{x_t - x^*}^2 - 2 \eta_t \inner{g_t}{x_t - x^*} + 2\eta_t \inner{\hat{z}_t}{x_t - x^*} + \eta_t^2 \norm{\hat{g}_t}^2\nonumber\\
        &~\leq~ \norm{x_t -x^*}^2 - 2\eta_t \left (f(x_t) - f(x^*)    \right ) - \eta_t \mu \norm{x_t - x^*}^2 \\& \qquad~+~ 2 \eta_t \inner{\hat{z}_t}{x_t - x^*}  + \eta_t^2 \norm{\hat{g}_t}^2 \nonumber\\
        &~\leq~ \left( 1 - 2\eta_t \mu \right ) \norm{x_t - x^*}^2 + 2 \eta_t \inner{\hat{z}_t}{x_t - x^*} + \eta_t^2 \norm{\hat{g}_t}^2 \nonumber\\
        &~=~ \EquationName{rakhlinEq} \left ( 1 - \frac{4}{t+1}\right )\norm{x_t - x^*}^2 + \frac{4}{\mu (t+1)} \inner{\hat{z}_t}{x_t - x^*} + \frac{4}{ \mu^2 (t+1)^2} \norm{\hat{g}_t}^2.
\end{align}Repeatedly applying \Equation{rakhlinEq} until $t=4$, yields the following
\begin{align}
    \norm{x_{t+1} - x^*}^2 &~\leq~ \frac{4}{\mu} \sum_{i=4}^t \left [ \frac{1}{i+1} \prod_{j=i+1}^t \left (  1 - \frac{4}{j+1}\right ) \right ] \cdot \inner{\hat{z}_i}{x_i - x^*}  \nonumber\\ &\quad+~\frac{4}{\mu^2} \sum_{i=4}^{t} \left [  \frac{1}{(i+1)^2}  \prod_{j=i+1}^t \left ( 1 - \frac{4}{j+1}\right ) \right ] \cdot \norm{\hat{g}_t}. 
\end{align}Observing that 
\[ \prod_{j=i+1}^t \left ( 1 - \frac{4}{j+1} \right ) ~=~ \prod_{j=i+1}^t\frac{j-3}{j+1} ~=~ \frac{ (i-2) \cdot (i-1)\cdot i \cdot (i+1)}{(t-2) \cdot (t-1)\cdot t \cdot (t+1)},   \]proves the lemma by taking $a_i(t) = \frac{1}{i+1} \cdot \frac{ (i-2)\cdot (i-1)\cdot i \cdot (i+1)}{(t-2) \cdot (t-1)\cdot t \cdot (t+1)}$ and $b_i(t)  = \frac{1}{(i+1)^2} \cdot \frac{(i-2) \cdot (i-1)\cdot i \cdot (i+1)}{(t-2) \cdot (t-1)\cdot t \cdot (t+1)}$ 
\end{proofof}
\section{Proof of high probability lower bound}
\SectionName{lb_proof}

In this section we show that the error of SGD when returning $\sum_{t=1}^T \frac{t}{T(T+1)/2} x_t$ is $\Omega( \log(1/\delta)/T )$ with probability at least $\delta$. We begin by stating a useful lemma.

\begin{lemma}[{\cite[Lemma~4]{RevChernoff}}]
\LemmaName{ReverseChernoff}
Let $X_1, \ldots, X_n$ be independent random variables taking value $\{-1, +1\}$ uniformly at random and $X = \frac{1}{n} \sum_{i=1}^n X_i$.
Suppose $\sqrt{6} \leq c \leq \sqrt{n}/2$, then
\[
\prob{X \geq  \frac{c}{\sqrt{n}} } \geq \exp(-9c^2/2).
\]
\end{lemma}

\begin{proofof}{\Claim{LB}}
Since the gradient oracle satisfies the assumption in \Claim{lowerBoundIterates} (which we prove below), we may apply \Claim{lowerBoundIterates} to obtain:
\begin{align*}
    \sum_{t=1}^T \gamma_t x_t  
        &~=~ \sum_{t=2}^T \gamma_t \left [\frac{1}{t}\sum_{i=1}^{t-1} \hat{z}_i \right ] \qquad\text{(by \Claim{lowerBoundIterates})}\\
        &~=~ \frac{2}{T(T+1)}\sum_{t=2}^T \sum_{i=1}^{t-1} \hat{z}_i \qquad\text{(definition of $\gamma_t$)}\\
        &~=~ \frac{2}{T(T+1)}\sum_{i=1}^{T-1} \hat{z}_i \cdot \left ( T - i \right ) \qquad\text{(swap order of summation)}\\
        &~=~ \frac{2}{T(T+1)}\sum_{i=T/2+1}^{3T/4}\hat{z}_i \cdot \left (T-i\right) \qquad\text{($\hat{z}_i = 0$ for all other $i$)}\\
        &~=~ \frac{2}{4} \left (\frac{1}{T/4}\sum_{i=T/2+1}^{3T/4} X_i \right ) \qquad\text{(definition of $\hat{z}_i$)}.
\end{align*}Now, we may apply \Lemma{ReverseChernoff} with $c = \frac{\sqrt{2\log (1/\delta)}}{3}$, and $n = T/4$ to obtain: 
\[  f \left ( \sum_{t=1}^T \gamma_t x_t  \right ) ~=~ \frac{1}{2} \left ( \sum_{t=1}^T \gamma_t x_t \right )^2 ~\geq~ \frac{1}{2}\left ( \frac{1}{2} \frac{\sqrt{2 \log(1/\delta)}}{3 \sqrt{T/4}} \right )^2 ~=~ \frac{\log(1/\delta)}{9 \cdot T},  \]with probability at least $\delta.$
\end{proofof}

The proof of \Claim{LB} required the use of \Claim{lowerBoundIterates}. We now provide a proof of this claim

\begin{proofof}{\Claim{lowerBoundIterates}}
We prove the claim via induction. For the base case consider $x_2 = \Pi_{\cX} \left (x_1 -\eta_1 \hat{g}_1 \right)$. Recall that $\hat{g}_1 = g_1 - \hat{z}_1$ where $g_1$ is the gradient of $f$ at $x_1$. Since $x_1 = 0$, we have $g_1 = 0$ and $x_2 = \Pi_{\cX} \left (\eta_1 \hat{z}_1 \right) = \frac{1}{2}\hat{z}_1$ because $\abs{\hat{z}_t} \leq 1$ for all $t$ and $\eta_t = \frac{1}{t+1}$.

Next, assume that $x_t = \frac{1}{t}\sum_{i=1}^{t-1}\hat{z}_i$. Then, $x_{t+1} = \Pi_{\cX}(y_t)$ where $y_t = x_t - \eta_t \hat{g}_t$ where $\hat{g}_t$ = $\nabla f (x_t) - \hat{z}_t.$ Hence, we have
\[
y_{t} ~=~ \frac{1}{t}\sum_{i=1}^{t-1} \hat{z}_t - \eta_t \left (\frac{1}{t}\sum_{i=1}^{t-1} \hat{z}_i - \hat{z}_t   \right ) ~=~ \frac{1}{t+1} \sum_{i=1}^t\hat{z}_t.
\] Clearly, $y_t \in \cX$, and therefore $x_{t+1} = y_t = \frac{1}{t+1} \sum_{i=1}^{t}\hat{z}_t$ as desired.

\end{proofof}
\section{Subgaussian noise extension}

The main result in this section is a strengthening of \Theorem{StochasticStrongPGD} by weakening the bounded noise assumption on the stochastic gradient oracle. First, we require a definition.

\begin{definition}
A random variable $X$ is said to be $\kappa$-subgaussian if $\expect{\exp \left ( X^2/\kappa^2 \right ) } \leq 2$.
In additional, we say that $X$ is $\kappa$-subgaussian conditioned on $\cF$ if $\expect{\exp \left ( X^2/\kappa^2 \right ) } \leq 2$.
Note that $\kappa^2$ in this setting may itself be a random variable. 
\end{definition}
\begin{remark}
Note that the class of subgaussian random variables contains bounded random variables. Furthermore, this class also contains Gaussian random variables (which, of course, are not bounded). Therefore, the following theorem is indeed a strengthening of \Theorem{StochasticStrongPGD}, which only dealt with stochastic gradient oracles that used almost surely bounded noise.
\end{remark}
\begin{theorem}
\TheoremName{StochasticStrongPGDSubGaussian}
Let $\cX \subseteq \bR^n$ be a convex set.
Suppose that $f : \cX \rightarrow \bR$ is $\mu$-strongly convex
(with respect to $\norm{\cdot}_2$) and $L$-Lipschitz.
Assume that:
\begin{enumerate}[label=(\emph{\alph*})]
\item $g_t \in \partial f(x_t)$ for all $t$ (with probability $1$).
\item $\norm{\hat{z}_t}$ is $\kappa$-subgaussian conditioned on $\cF_{t-1}$ for some $\kappa \in \bR$.
\end{enumerate}
Set $\eta_t = \frac{2}{\mu (t+1)}$.
Let $\gamma_t = \frac{t}{T(T+1)/2}$.
Then, for any $\delta \in (0,1)$, with probability at least $1-\delta$ we have,
$$
f\left( \sum_{t=1}^T \gamma_t x_t \right) - f(x^*)
    ~\leq~  O \left (\frac{(L+\kappa)^2}{ \mu}  \cdot \frac{\log \left (1 / \delta \right )}{T}  \right ). 
$$
\end{theorem}

\begin{proofof}{\Theorem{StochasticStrongPGDSubGaussian}}

We may follow the proof of \Theorem{StochasticStrongPGD} from \Section{ub} and remove any bound used on $\norm{\hat{g}_t}^2$ to obtain 
\begin{equation}
\EquationName{subGaussianNoiseMain}
    f\left (\sum_{t=1}^T \gamma_t x_t    \right ) - f(x^*) ~\leq~ \frac{2}{T(T+1)} \underbrace{\sum_{t=1}^T t \cdot \inner{\hat{z}_t}{x_t - x^*}}_{:= Z_T}  + \frac{2}{\mu T (T+1)} \sum_{t=1}^T\norm{\hat{g}_t}^2.
\end{equation}
\Theorem{StochasticStrongPGDSubGaussian} follows trivially from the following two lemmata:
\begin{lemma}
\LemmaName{subExpSum}
For any $\delta \in (0,1)$, $\sum_{t=1}^T \norm{\hat{g}_t}^2 = O \left ( \left (L + \kappa \right )^2 T \cdot \log (1/\delta)\right )$ with probability at least $1-\delta.$
\end{lemma}
\begin{lemma}
\LemmaName{boundZsubgaussian}
Let $Z_T = \sum_{t=1}^T t \cdot \inner{\hat{z}_t}{x_t - x^*}.$ Then, for any $\delta \in (0,1)$, we have $ Z_T = O \left (\frac{(L + \kappa)^2}{\mu} T\cdot  \log(1/\delta)\right )$ with probability at least $1 -\delta.$
\end{lemma}
\end{proofof}

\subsection{Proof of \Lemma{subExpSum}}

We begin with a fact about subgaussian random variables.

\begin{claim}
\ClaimName{psi2norm}
Let X be a random variable. Define $\norm{X}_{\psi_2}$ as $\inf \setst{ t >0 }{ \expect{ \exp \left ( X^2 / t^2 \right )} \leq 2 }.$ Then, $\norm{\cdot }_{\psi_2}$ is a norm. 
\end{claim}

Observe that $X$ is $\kappa$-subgaussian if and only if $\norm{X}_{\psi_2} \leq \kappa$. As a consequence of \Claim{psi2norm}, we have the following claim.

\begin{claim}
\ClaimName{Gpsi2}
There exists $\xi = O(L + \kappa)$, such that  $\norm{\hat{g}_t}$ is $\xi$-subgaussian conditioned on $\cF_{t-1}$. 
\end{claim}

\begin{proof}
Using the triangle inequality, we have 
\[  \norm{\hat{g}_t} ~=~ \norm{g_t -\hat{z}_t } ~\leq~ \norm{g}_t + \norm{\hat{z}_t}.  \]Therefore, 
\[ \norm{\norm{\hat{g}_t} \mid  \cF_{t-1}}_{\psi_2} ~\leq~ \norm{  \norm{g_t} \mid \cF_{t-1}}_{\psi_2}  + \norm{\norm{\hat{z}_t}\mid \cF_{t-1}}_{\psi_2} ~\leq~  \norm{  \norm{g}_t \mid \cF_{t-1
}}_{\psi_2} + k, \]
because we assumed $\norm{\hat{z}_t}$ is conditionally $\kappa$-subgaussian. Also, note that $\norm{g_t}$ is conditionally $(L/\ln 2)$-subgaussian because $f$ is $L$-Lipschitz and so $\norm{g_t} \leq L$. 
\end{proof}
Now, we proceed to prove \Lemma{subExpSum} using an MGF bound:

\begin{claim}There exists $\xi = O\left ( L + \kappa \right )$ such that 
\ClaimName{MGFBound}
$\expect{\exp \left ( \sum_{i=1}^T \norm{\hat{g}_i}^2  / \left (T \cdot \xi^2 \right ) \right ) } \leq 2.$
\end{claim} 

Using \Claim{MGFBound} we can prove \Lemma{subExpSum}:

\begin{proofof}{\Lemma{subExpSum}}
Using the exponentiated Markov inequality we have for any $\lambda > 0$:
\[ \prob{  \sum_{t=1}^T  \norm{\hat{g}_t}^2 \geq x   } ~\leq~  \frac{\expect{  \exp \left ( \lambda \sum_{t=1}^T  \norm{\hat{g}_t}^2  \right )  }}{\exp \left ( \lambda x\right) }.  \]Plugging in $\lambda  = O \left ( \frac{1}{T \cdot\xi^2} \right )$ and $x = O \left (T \cdot\xi^2 \log (1/\delta) \right)$ completes the proof.
\end{proofof}

It remains to prove \Claim{MGFBound}.

\begin{proofof}{\Claim{MGFBound}}
We will show that for every $1 \leq t \leq T,$
\begin{align}
\EquationName{Intermediate1}
    \expect{ \exp \left (  \sum_{i=1}^t \norm{\hat{g}_i}^2/(T \cdot \xi^2 ) \right )  } ~\leq~ 2^{1/T} \expect{ \exp \left (  \sum_{i=1}^{t-1} \norm{\hat{g}_i}^2/(T \cdot \xi^2 ) \right )  }.
\end{align} 
Indeed, 
\begin{align*}
    &\expect{ \exp \left (  \sum_{i=1}^t \norm{\hat{g}_i}^2/(T \cdot \xi^2 ) \right )  }  \\
    &=~  \expect{\exp \left (\sum_{i=1}^{t-1} \norm{\hat{g}_i}^2/(T \cdot \xi^2 ) \right ) \expectg{\exp \left ( \norm{\hat{g}_t}^2/(T \cdot \xi^2 ) \right ) }{\cF_{t-1}}    }.
\end{align*}Furthermore, 
\[  \expectg{\exp \left ( \norm{\hat{g}_t}^2/(T \cdot \xi^2 ) \right ) }{\cF_{t-1}} \leq 2^{1/T} \] for all $1 \leq t \leq T$ using \Claim{Gpsi2} and Jensen's inequality. Therefore, \Equation{Intermediate1} is true for all $1\leq t \leq T$.

Hence, 
\[\expect{ \exp \left (  \sum_{i=1}^t \norm{\hat{g}_i}^2/(T \cdot \xi^2 ) \right )  } ~\leq~ \left ( 2^{1/T}\right )^T ~=~ 2, \]as desired.
\end{proofof}

\subsection{Proof of \Lemma{boundZsubgaussian}}

We may follow the proof of \Lemma{StronglyConvexSGDBound1} from \Subsection{BoundingZ}. Define $d_t = t \cdot \inner{\hat{z}_t}{x_t -x^*},$ $
\tilde{v}_{t-1} := 2\kappa^2 \cdot t^2 \norm{x_t-x^*}^2,$ and $\tilde{V}_T = \sum_{t=1}^T \tilde{v}_{t-1}.$ Note that $\tilde{v}_{t-1}$ is $\cF_{t-1}$-measurable.  

\begin{claim}
\ClaimName{subGaussianMGFforFan}
For all $t$ and $\lambda > 0,$ we have 
\[  \expectg{\exp \left (\lambda d_t    \right )}{\cF_{t-1}} \leq \exp \left ( \frac{\lambda^2}{2} \tilde{v}_{t-1} \right ). \]
\end{claim}

The proof of this requires a lemma from \cite{Ver18}.

\begin{lemma}[{\cite[Proposition 2.5.2]{Ver18}}]
\LemmaName{MeanZeroMGFTrick}
Suppose $X$ is a mean-zero random variable such that  $\expect{\exp \left (  X/\kappa^2 \right )} \leq 2.$ Then, $\expect{\exp(\lambda X )} \leq \exp \left ( \lambda^2 \kappa^2 \right )$ for all $\lambda > 0$.
\end{lemma}

\begin{proofof}{\Claim{subGaussianMGFforFan}}
Because $\norm{\hat{z}_t}$ is $\kappa$-subgaussian conditioned on $\cF_{t-1},$ we have by Cauchy-Schwarz
\[  \expectg{ \exp \left ( \frac{t^2 \cdot \inner{ \hat{z}_t}{x_t- x^*}^2}{\kappa^2 \cdot t^2 \norm{x_t - x^*}^2} \right ) }{\cF_{t-1}} ~\leq~ \expectg{ \exp \left( \norm{\hat{z}_t}^2 / \kappa^2 \right ) }{\cF_{t-1}} ~\leq~ 2.  \]Therefore, by \Lemma{MeanZeroMGFTrick} we have 
\[  \expectg{\exp \left ( \lambda \right )}{\cF_{t-1}} ~\leq~ \exp \left ( \frac{\lambda^2}{2} (2\kappa^2 \cdot t^2 \norm{x_t - x^*}^2 \right ), \]as desired.
\end{proofof}

To bound $Z_T,$ we will proceed similarly as in \Subsection{BoundingZ}. We will bound the TCV of $Z_T$ by a linear combination of the increments. The only difference is, we will show that this bound holds with high probability, instead of with probability one. This will allow us to use a form of the Generalized Freedman Inequality (\Theorem{FanV2.1}) which the case where we can bound the total conditional variance by a linear transformation of the increments of the martingale with high probability.

\begin{lemma}
\LemmaName{ChickenAndEggSubGaussian}
There exists non-negative constants $\alpha_1, \ldots, \alpha_T = O\left((L +\kappa)^2\frac{T}{\mu}\right)$ and $\beta = O \left( \frac{(L+\kappa)^4}{\mu^2}T^2 \right )$ such that for every $\delta \in (0,1)$,
$\tilde{V}_T \leq \sum_{t=1}^T\alpha_t d_t + \beta \log(1/\delta)$ with probability at least $1 - \delta$.
\end{lemma}

Given \Lemma{ChickenAndEggSubGaussian}, we are ready to prove \Lemma{boundZsubgaussian}. But first, we require a slightly more general version of the Generalized Freedman Inequality where the bound on the TCV by a linear transformation of the increments of the martingale holds only with arbitrarily high probability, rather than with probability 1.

\begin{theorem}[Generalized Freedman, {\cite[Theorem~3.3]{HLPR18}}]
\TheoremName{FanV2.1}
Let $\{ d_t, \cF_t \}_{t=1}^T$ be a martingale difference sequence.
Suppose that, for $t \in [T]$, $v_{t-1}$ are non-negative  $\cF_{t-1}$-measurable random variables satisfying $\expectg{\exp(\lambda d_t)}{\cF_{t-1}} \leq \exp\left(\frac{\lambda^2}{2} v_{t-1}\right)$ for all $\lambda > 0$.
Let $S_T = \sum_{t=1}^T d_t$ and $V_T = \sum_{t=1}^T v_{t-1}$.
Suppose there exists $\alpha_1, \ldots, \alpha_T, \beta \in \bR_{\geq 0}$ such that for every $\delta \in (0,1)$, $V_T \leq \sum_{t=1}^T \alpha_t d_t + \beta \log(1/\delta)$.
Let $\alpha \geq \max_{t \in [T]} \alpha_t$.
Then
\[
\prob{S_T \geq x} ~\leq~ \exp\left( -\frac{x^2}{4\alpha\cdot x + 8\beta } \right) + \delta.
\]

\end{theorem}

\begin{proofof}{\Lemma{boundZsubgaussian}}
By \Claim{subGaussianMGFforFan} we have $\expectg{\exp \left ( \lambda d_t \right )}{\cF_{t-1}} \leq \exp\left ( \frac{\lambda^2}{2} \tilde{v}_{t-1} \right ).$ By \Lemma{ChickenAndEggSubGaussian} we have that for every $\delta \in (0,1)$ $\tilde{V}_{T}\leq  \sum_{t=1}^T\alpha_t d_t + \beta \log(1/\delta),$ with probability at least $1 -\delta.$ Plugging $\alpha = O \left ((L+\kappa)^2\frac{T}{\mu} \right),$ $\beta = O\left ( \frac{ (L+\kappa)^4}{\mu^2} T^2 \log(1/\delta) \right )$ and $x = O \left ( \frac{(L+\kappa)^2}{\mu} \cdot T \log(1/\delta) \right)$ into \Theorem{FanV2.1}, proves \Lemma{boundZsubgaussian}. 
\end{proofof}

It remains to prove \Lemma{ChickenAndEggSubGaussian}. 

\begin{proofof}{\Lemma{ChickenAndEggSubGaussian}}
Observe that $\tilde{V}_T = 2\kappa^2 V_T = \sum_{t=1}^T t^2 \cdot \norm{x_t - x^*}^2$ where $V_T$ was defined in \Subsection{BoundingZ}. We focus our attention on bounding $V_T$, and then scale up accordingly at the end.

We may follow the proof of \Lemma{BoundingB} with a key modification: Do not bound $\norm{\hat{g}_i}$ by $L+1$ as this is no longer valid, because we no longer are using the bounded noise assumption.

This yields:

\begin{align}
\EquationName{VTSubgaussian}
    &V_T \\
    &~\leq~ \sum_{i=3}^{T-1}  \underbrace{\frac{4}{\mu}\left( \sum_{t=i+1}^T t^2 \frac{a_i(t-1)}{i}   \right )}_{\coloneqq \alpha_i} \cdot i \cdot \inner{\hat{z}_i}{x_i - x^*} + \frac{4}{\mu^2}\underbrace{\sum_{t=4}^T t^2 \left (  \sum_{i=3}^{t-1}b_i(t-1)  \norm{\hat{g}_i}^2\right )
    }_{\coloneqq G_T} + \frac{56L^2}{\mu^2}.
\end{align}Define $\alpha_1,\alpha_2, \alpha_T = 0.$ We already showed in the proof of \Lemma{BoundingB} that $\alpha_i = O\left( \frac{T}{\mu} \right).$ Therefore, it remains to bound $G_T$ by $O\left ( (L + \kappa)^2T^2 \cdot \log(1/\delta) \right ),$ with probability at least $1-\delta.
$ We rewrite $G_T$ as 
\[ G_T ~=~ \sum_{i=3}^{T-1} \underbrace{\left ( \sum_{t=i+1}^T t^2 b_i(t-1) \right )}_{\coloneqq s_i}\cdot \norm{\hat{g}_i}^2 ~=~ \sum_{i=3}^{T-1} s_i \norm{\hat{g}_i}^2 .  \]

We use the following MGF bound on $G_T$, which we prove below.

\begin{claim}
\ClaimName{GTMGFBound}
$\expect{\exp \left ( \lambda G_T \right )} \leq \exp \left ( \lambda O\left (\xi^2\right) \sum_{i=3}^{T-1 }s_i   \right )$ for all $\lambda = O \left( \frac{1}{\xi^2 \max \set{s_i}} \right)$.
\end{claim}
Therefore, via an exponentiated Markov inequality and \Claim{GTMGFBound}, we have
\[  \prob{G_T \geq x} ~\leq~ \frac{\expect{\exp \left ( \lambda G_T \right )}}{\exp\left ( \lambda x \right )} ~\leq~ \exp \left (   \lambda O\left (\xi^2\right )\sum_{i=3}^{T-1} s_i - \lambda x \right ).\]
Setting $\lambda = O \left ( \frac{1}{\xi^2 \sum_{i=3}^{T-1}s_i}  \right )$ and $x = O \left(\xi^2 \sum_{i=3}^{T-1}s_i \cdot \log(1/\delta)\right )$ shows $G_T \leq  O \left(\xi^2 \sum_{i=3}^{T-1}s_i \cdot \log(1/\delta)\right )$ with probability at least $1-\delta.$ Observe that $s_i = O(T):$
\[  s_i ~=~ \sum_{t=i+1}^T t^2 b_i(t-1) ~=~ \sum_{t=i+1}^T t^2 O \left ( \frac{i^2}{t^4} \right ) ~=~ \sum_{t=i+1}^T O\left(1\right) ~=~ O\left (T\right). \]
Therefore $x = O \left ( \xi^2 T^2 \log(1/\delta) \right ) = O\left(  (L + \kappa)^2 T^2 \log(1/\delta)\right ),$ with probability at least $1-\delta.$ That is, with probability at least $1-\delta$
\[  G_T \leq O \left ( \left ( L + \kappa \right )^2 T^2 \log (1/\delta) \right ). \]

Plugging this back in to \Equation{VTSubgaussian}, we obtain 
\[  V_T ~\leq~ \sum_{i=1}^{T} \alpha_i d_i + \beta \log(1/\delta)  \]where $\alpha_i = O( \frac{T}{\mu})$ and $\beta = O \left (   \frac{(L + \kappa)^2}{\mu^2} T^2 \cdot \log(1/\delta) \right )$. Multiplying both sides by $2\kappa^2$ yields the desired bound on $\tilde{V}_T.$

\end{proofof}

Now it remains to prove \Claim{GTMGFBound}.
\begin{proofof}{\Claim{GTMGFBound}}
We will show that for every $t$ and for all $\lambda \leq 1/\max \set{s_i}$, 
\[  \expect{\exp \left ( \lambda \sum_{i=3}^t s_i \norm{\hat{g}_i}^2  \right) } \leq \exp \left ( \lambda O \left ( \xi^2\right )s_t \right)\expect{\exp\left (\lambda \sum_{i=3}^{t-1} s_i \right ) }. \]The claim then follows by recursively applying the above inequality. Note that $s_i$ is $\cF_{i-1}$ measurable. So, we have
\begin{align*}
    \expect{\exp \left ( \lambda \sum_{i=3}^t s_i \norm{\hat{g}_i}^2  \right) } 
        &~=~ \expect{\exp \left ( \lambda\sum_{i=3}^{t-1} s_i \norm{\hat{g}_i}^2  \right ) \expectg{\exp\left (\lambda s_t \norm{\hat{g}_t}^2  \right )}{\cF_{t-1}}}.
\end{align*}Note that because $\norm{ \norm{\hat{g}_t} \mid \cF_{t-1}}_{\psi_2} \leq \xi,$ this implies 
\[  \expectg{\exp\left (\norm{\hat{g}_t}^2/\xi^2\right )}{\cF_{t-1}} \leq 2. \]Therefore, if $\lambda = O \left ( \frac{1}{\xi^2} \right ).$ Then by Jensens inequality, raising both sides of the above inequality to the power of $\lambda \xi^2$ yields
\[  \expectg{ \exp \left ( \lambda  \norm{\hat{g}_t}^2 \right ) }{\cF_{t-1}}  \leq \exp \left ( \lambda O \left ( \xi^2\right ) \right ).\] Hence, if $\lambda = O\left ( \frac{1}{\xi^2 \max\set{s_i}} \right ),$ then for every $t$ we have
\[ \expectg{\exp\left (\lambda s_t \norm{\hat{g}_t}^2  \right )}{\cF_{t-1}} \leq \exp \left (\lambda O \left ( \xi^2\right )s_t \right ),   \]which completes the proof.
\end{proofof}
\section{Additional experiments}
\SectionName{additional_experiments}

In each experiment we run SGD for the regularized SVM optimization problem described in \Section{experimental_results}. We use regularization parameter $\lambda = 1/n$ and step size $\eta_t = \frac{2}{\mu (t+1)}$.  For each return strategy, we run many trials of SGD and plot the objective value over time for every trial. At any point in time, the darkness of the plot at a specific objective value indicates the number of trials that achieved that value at that time.   

We use the freely available data sets \emph{quantum} ($m=50000$ and $n=78$), \emph{covtype} ($m = 581012$ and $n=54$) and \emph{rcv1} ($m = 20242$ and $n = 47236$). We run 1000 trials of SGD on the \emph{quantum} data set, 80 trials of SGD on the \emph{covtype} data set and 70 trials of SGD on the \emph{rcv1} data set. The \emph{quantum} data set can be found at the \href{http://osmot.cs.cornell.edu/kddcup/datasets.html}{\color{blue}{KDD cup 2004 website}}
and \emph{covtype} and \emph{rcv1} can be found at the \href{https://www.csie.ntu.edu.tw/~cjlin/libsvmtools/datasets/}{\color{blue}{LIBSVM website}}.
\begin{landscape}
\begin{figure}[h]
    \centering
    \begin{subfigure}[b]{\paperwidth}
        \centering
        \includegraphics[width=\paperwidth]{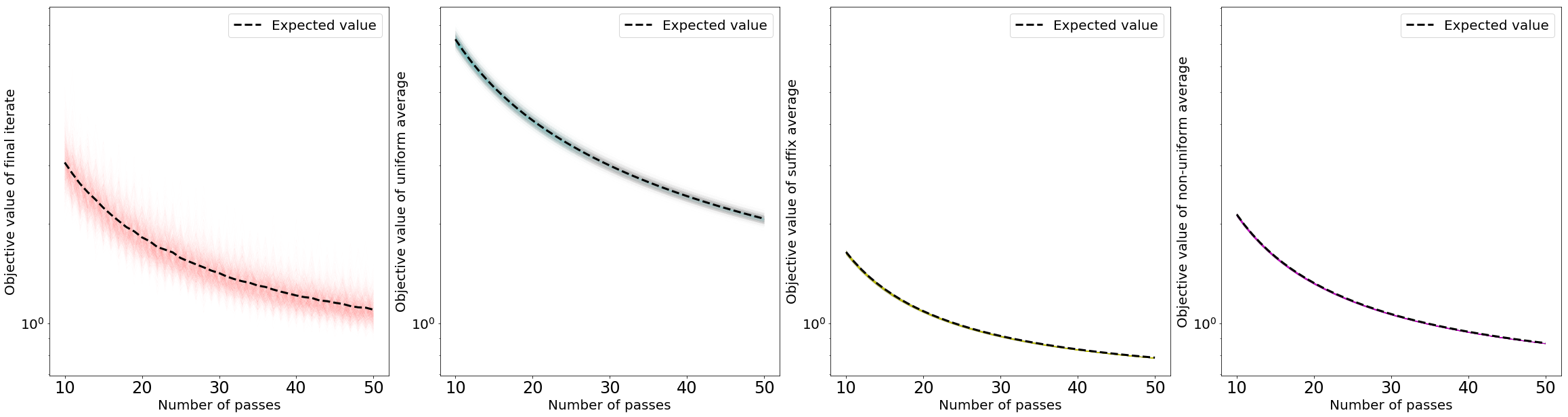}
    \caption{quantum}
    \FigureName{quantum}
    \end{subfigure} \\
    \begin{subfigure}[b]{\paperwidth}
        \centering
        \includegraphics[width=\paperwidth]{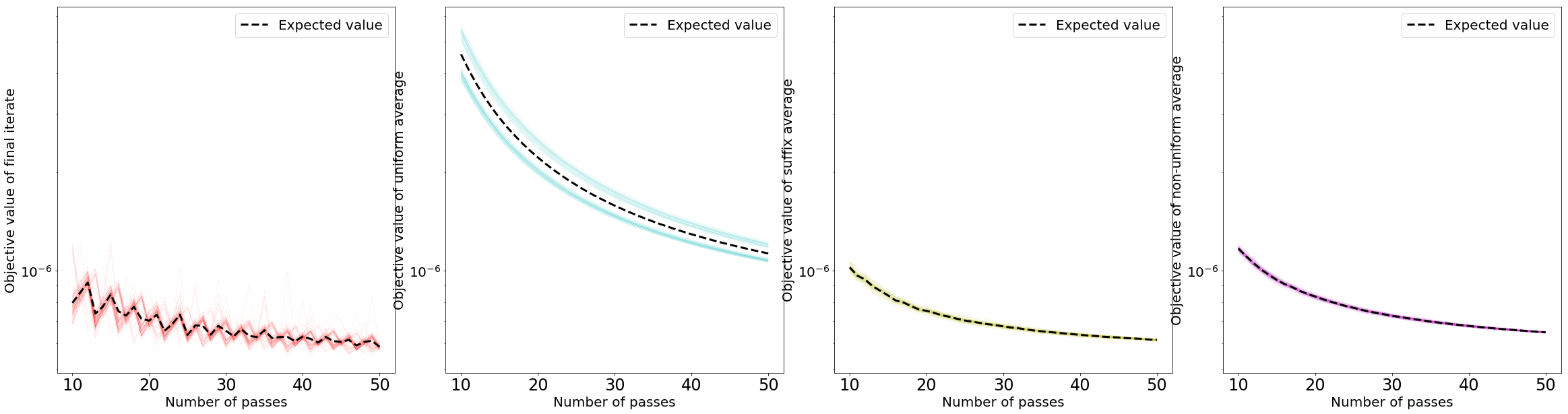}
    \caption{covtype (80 trials)}
    \FigureName{covtype}
    \end{subfigure}
\end{figure}
\begin{figure}
    \ContinuedFloat
    \begin{subfigure}[b]{\paperwidth}
        \centering
        \includegraphics[width=\paperwidth]{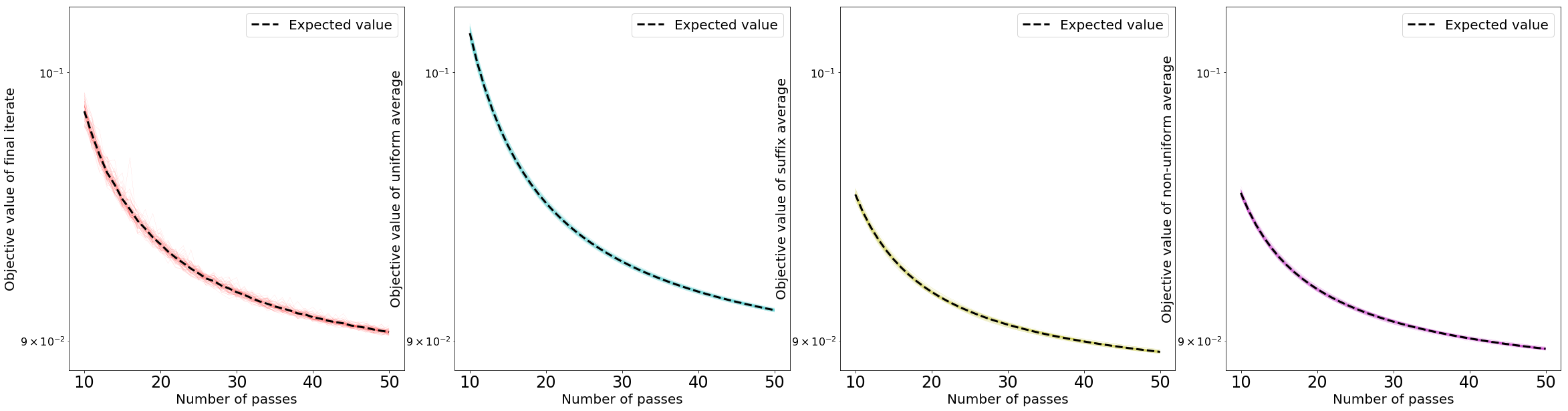}
    \caption{rcv1 (70 trials)}
    \FigureName{rcv}
    \end{subfigure}
    \caption{Number of effective passes vs. objective value. \Figure{quantum} plots the results for the \emph{quantum} dataset; \Figure{covtype} plots the results for \emph{covtype} dataset; \Figure{rcv} plots the results for the \emph{rcv1} dataset. From left to right, we plot the objective value over time of the final iterate, uniform average, suffix average and non-uniform average for 1000 trials of SGD.}
    \label{fig:iterates_additional}
\end{figure}
\end{landscape}
\end{document}